%% file: main.tex
\documentclass{article}
\usepackage{ijcai25}
\usepackage{times}
\usepackage{soul}
\usepackage{url}
\usepackage[hidelinks]{hyperref}
\usepackage[utf8]{inputenc}
\usepackage[small]{caption}
\usepackage{graphicx}
\usepackage{amsmath}
\usepackage{amsthm}
\usepackage{booktabs}
\usepackage{algorithm}
\usepackage{algorithmic}
\usepackage[switch]{lineno}

\input{packages}

\input{commands}

\title{\nosLong:\\ Solving POMDPs Sans Numerical Optimisation\footnote{Technical details and proofs are contained in the Supplementary Material (\url{https://github.com/RDLLab/pomdp-py-porpp}).}}

\author{
Edward Kim, and
Hanna Kurniawati\\
Australian National University\\
\{edward.kim, hanna.kurniawati\}@anu.edu.au
}

\begin{document}

\maketitle

\begin{abstract}
    This paper proposes \nosLong, a novel anytime online approximate \pomdp~solver which samples meaningful future histories very deeply while simultaneously forcing a gradual policy update. We provide theoretical guarantees for the algorithm's underlying scheme which say that the performance loss is bounded by the \emph{average} of the sampling approximation errors rather than the usual maximum; a crucial requirement given the sampling sparsity of online planning. Empirical evaluations on two large-scale problems with dynamically evolving environments---including a helicopter emergency scenario in the Corsica region requiring approximately 150 planning steps---corroborate the theoretical results and indicate that our solver considerably outperforms current online benchmarks.
\end{abstract}

\section{Introduction}

\input{introduction}

\section{Background and Related Work}

\input{background}

\section{\nos}
\label{sec.algo}

\input{algorithms}

\section{Experiments}

\input{experiments}

\section{Summary}

\input{summary}

\section*{Acknowledgments}

This work was supported by Safran Electronics \& Defense Australia Pty Ltd and Safran Group under the ARC Linkage project LP200301612.

\bibliographystyle{named}
\bibliography{references}

\end{document}

%% file: packages.tex
\usepackage[dvipsnames]{xcolor} 
\usepackage{amssymb} 
\usepackage{mathrsfs} 
\usepackage{subcaption} 
\usepackage{xspace}

%% file: commands.tex
\newtheorem{theorem}{Theorem}

\newtheorem{assumption}{Assumption}

\newcommand{\mdp}{MDP}

\newcommand{\pomdp}{POMDP}

\newcommand{\rbpomdpLong}{Reference-Based POMDP}
\newcommand{\rbpomdp}{RBPOMDP}

\newcommand{\nosLong}{Partially Observable Reference Policy Programming\xspace} 
\newcommand{\nos}{PORPP\xspace} 

\newcommand{\Reals}{\mathbb{R}} 
\newcommand{\Prob}{P} 
\newcommand{\KL}{\operatornamewithlimits{KL}} 
\newcommand{\argmax}{\operatornamewithlimits{arg\,max}} 
\newcommand{\argmin}{\operatornamewithlimits{arg\,min}} 

\newcommand{\dist}{\rho} 

\newcommand{\hist}{\ensuremath{h}} 
\newcommand{\bel}{\ensuremath{b}}
\newcommand{\belSpace}{\ensuremath{\mathcal{B}}} 
\newcommand{\reachable}{\ensuremath{\mathcal{R}}} 
\newcommand{\optReachable}{\ensuremath{\mathcal{R}^*}} 
\newcommand{\policies}{\ensuremath{\Pi}} 

\newcommand{\sampler}{\ensuremath{\tilde{\pi}}} 
\newcommand{\subsetBA}{\ensuremath{\Omega}} 

\newcommand{\States}{\ensuremath{\mathcal{S}}} 
\newcommand{\Actions}{\mathcal{A}} 
\newcommand{\Observations}{\mathcal{O}} 
\newcommand{\TP}{\mathcal{T}} 
\newcommand{\OP}{\mathcal{Z}} 
\newcommand{\Reward}{R} 
\newcommand{\sta}{\ensuremath{{s}}} 
\newcommand{\nst}{\ensuremath{{s'}}} 
\newcommand{\act}{\ensuremath{{a}}} 
\newcommand{\mact}{\ensuremath{\textbf{a}}} 
\newcommand{\obs}{\ensuremath{{o}}} 
\newcommand{\mobs}{\ensuremath{\textbf{o}}} 
\newcommand{\discount}{{\gamma}} 
\newcommand{\initBel}{{\bel_0}} 

\newcommand{\QVal}{Q} 
\newcommand{\optQVal}{\QVal^*} 
\newcommand{\pol}{{\pi}} 
\newcommand{\optPol}{{\pol}^*} 

\newcommand{{\VF}}{\mathscr{V}} 
\newcommand{{\QF}}{\mathscr{Q}} 

\newcommand{\Rmax}{\Reward_{\max}} 
\newcommand{\Vmax}{\QVal_{\max}} 

\newcommand{\covNum}{\mathcal{C}} 
\newcommand{\covNumInt}{\mathcal{C}^{\circ}} 
\newcommand{\pakNum}{\mathcal{P}} 
\newcommand{\covSet}{\mathcal{E}} 

\newcommand{\nearestBel}{\tilde{\tau}} 
\newcommand{\DPP}{\mathscr{L}} 
\newcommand{\logOp}{\mathcal{L}} 
\newcommand{\error}{\epsilon} 

\newcommand{\refPol}{\pol_0} 
\newcommand{\refVal}{\mathcal{V}} 
\newcommand{\temp}{\eta} 
\newcommand{\actPref}{\Psi} 
\newcommand{\approxAP}{\hat{\actPref}} 
\newcommand{\approxPol}{\hat{\pol}} 
\newcommand{\optRefPol}{{\pol}^*} 

\newcommand{\Tree}{{T}} 
\newcommand{\depth}{\text{depth}}
\newcommand{\GenModel}{\mathscr{G}}
\newcommand{\maxDepth}{D_\text{max}}

\newcommand{\AWF}{{\kappa_{\Actions}}} 
\newcommand{\AWA}{{\alpha_{\Actions}}} 

%% file: introduction.tex
Planning under uncertainty in non-deterministic and partially observable scenarios is critical for many (semi-)autonomous systems.
Such problems can be systematically framed as a Partially Observable Markov Decision Process (POMDP) \cite{kaebling98}.
Although solving infinite-horizon POMDPs in the worst case is undecidable \cite{madani}, the past decade has seen tremendous advances in the practicality of approximate POMDP solvers \cite{kurniawatiSurvey}. Most of these solvers are online sampling-based methods that numerically compute estimates of the expected total reward of performing different actions before optimising over these estimates. Due to difficulties in computing gradients, such solvers \emph{exhaustively enumerate} over the entire action space, which massively hinders fast computation of a close-to-optimal solution for problems with large action spaces and long horizons. This problem is even worse when the environment is also dynamically changing at each execution step.

The core difficulty is the curse of history where the set of possible futures branches by the size of the action space and grows exponentially with respect to the horizon.
Most existing methods try to abstract the problem into a simpler one by either reducing the size of the action space~\cite{wang1} or relying on \emph{macro actions}---i.e. a set of open-loop action sequences---to reduce the planning horizon~\cite{theo2003,he10:puma,KDHL2011,FRF2020,lee.rss}.
Still, the fundamental problem---i.e. exhaustive action enumeration---remains.

Recently, \cite{kkk23,lktkkk24} have softened this requirement by introducing the notion of a Reference-Based \pomdp~(\rbpomdp) which is a reformulation of a \pomdp~whose objective is penalised by the Kullback-Leibler (KL) divergence between a chosen and nominal \emph{reference policy}.
As such, a solution can be viewed as a KL-regularised improvement of the reference policy.
The form of objective allows analytical action optimisation so that the value can be approximated by estimating expectations under the reference policy.
This property accommodates solvers that have been shown to perform effectively on certain long-horizon tasks.
However, the \rbpomdp~formulation comes at the cost that the solution has a baked-in commitment to the reference policy.
In general, it is unclear a priori which reference policies yield near optimal policies for the original \pomdp~of interest, so the performance of the computed solution is vulnerable to mis-specification.

The aim of this paper is to build on the advantages of the \rbpomdp~framework while, in tandem, bolstering any vulnerabilities to mis-specification. To this end, our contribution is
an exact iterative scheme (Sect. \ref{subsec.exact.scheme}) whose successive policies can be viewed as solutions of a sequence of \rbpomdp s---i.e. KL-constrained policy improvements.
Theoretical analysis shows that the performance loss of the exact scheme is bounded by the \emph{average} of the sampling errors, which means the algorithm is less sensitive to large approximation errors (Theorem \ref{thm.approx.bound}).
We also contribute an explicit approximate scheme (Sect. \ref{subsec.approx.scheme}) and provide a \pomdp-specialised high-probability bound for the performance loss (Theorem \ref{thm.exp.bd}).
Finally, the scheme is practically implemented in our proposed algorithm \nosLong~(\nos)---an anytime online \pomdp-solver---and tested on two non-trivial long-horizon \pomdp s, one of which has a dynamically evolving environment.
Experimental results indicate that \nos~substantially outperforms current state-of-the-art online \pomdp~benchmarks.

%% file: background.tex
\subsection{\pomdp~Preliminaries}
\label{sec.pomdp}

An infinite-horizon \pomdp~is defined as the tuple $\langle \States, \Actions, \Observations, \OP, \TP, \Reward, \discount, \initBel \rangle$ where the sets of all possible agent states, actions and observations are denoted by $\States$, $\Actions$ and $\Observations$ respectively.
For clarity, our presentation is for countable sets.
The \emph{transition model} $\TP$ is such that $\TP(\nst \, | \, \act, \sta)$ is the conditional probability that $\nst \in \States$ occurs after performing $\act \in \Actions$ from $\sta \in \States$.
The \emph{observation model} $\OP$ is such that $\OP(\obs \, | \, \nst, \act)$ is the conditional probability that the agent perceives $\obs \in \Observations$ when it is in state $\nst \in \States$ after performing $\act \in \Actions$.
The \emph{reward model} is a real-valued function $\Reward: \States \times \Actions \rightarrow \Reals$ such that $| \Reward(\sta, \act) | \le \Rmax < \infty$ for all $\sta, \act$ and $\discount \in (0, 1)$ is the discount factor.

The agent does not know the true state, but it maintains a \emph{belief} about its state---i.e. a probability distribution $\bel$ on the space $\States$.
Let $\belSpace$ be the set of all such distributions.
Starting with a given initial belief $\initBel$, the agent's next belief $\bel'$ after taking the action $\act$ and perceiving some observation $\obs$ is given by
$
    \bel'(\nst) \propto \OP(\obs \, | \, \nst, \act) \sum_{\sta \in \States}  \TP(\nst \, | \, \act, \sta) \bel(\sta)
$
and we write $\bel' = \tau(\bel, \act, \obs)$ with the \emph{belief update operator} $\tau$.
We denote the set of \emph{reachable beliefs} by $\reachable_\initBel \subset \belSpace$; i.e. the set of beliefs reachable from $\initBel$ under some policy.
For any given belief $\bel$ and action $\act$ the expected immediate reward is given by
$
    \Reward(\bel, \act) := \sum_{\States} \Reward(\sta, \act) \bel(\sta)
$.
The probability that the agent perceives an observation $\obs \in \Observations$ having performed the action $\act \in \Actions$ under the belief $\bel$ is given by
\begin{equation}
    \Prob(\obs \, | \, \act, \bel) := \sum_{\nst \in \States} \OP(\obs \, | \, \nst, \act) \sum_{\sta \in \States} \TP(\nst \, | \, \act, \sta) \bel(\sta).
\end{equation}

A (stochastic) \emph{policy} is a mapping $\pol: \reachable_\initBel \rightarrow \Delta(\Actions)$. We denote its distribution for any given input $\bel \in \reachable_\initBel$ by $\pol(\cdot \, | \, \bel)$.
Let $\policies$ be the class of all stochastic policies.
For any $(\bel, \act) \in \reachable_\initBel \times \Actions$, define the \emph{action-value function} $\QVal^\pol: \reachable_\initBel \times \Actions \rightarrow \Reals$ recursively according to
\begin{multline} \label{eq.qval}
        \QVal^\pol(\bel, \act) = \Reward(\bel, \act) \\
        + \discount \sum_{\act', \obs} \QVal^\pol\big(\tau(\bel, \act, \obs), \act'\big) \Prob(\obs \, | \, \act', \bel) \, \pol(\act' \, | \, \bel).
\end{multline}
Given the reward is uniformly bounded, for any policy $\pol \in \policies$, we have $|\QVal^\pol(\bel, \act)| \le \Vmax := \Rmax/(1-\discount)$ for all pairs $(\bel, \act) \in \reachable_\initBel \times \Actions$.
A \emph{solution} to the \pomdp~is a policy $\optPol \in \policies$ satisfying
$
    \optQVal(\bel, \act) := \sup_{\pol \in \policies} \QVal^\pol(\bel, \act) = \QVal^{\optPol}(\bel, \act)
$
for all $(\bel, \act) \in \reachable_\initBel \times \Actions$.

\subsection{\pomdp~Packing and Covering Numbers}
\label{sec.pakcov}

For a Markov Decision Process (\mdp)~with finite state and action spaces, the usual input for complexity is the set cardinality $|\States| |\Actions|$ where it is generally assumed that the spaces are finite.
However, for the \pomdp, the reachable belief space $\reachable_\initBel$ is an uncountable subset even if $\States$ is finite so the notion of set cardinality is no longer a sensible complexity input.
A more reasonable approach is to choose a metric in $\Reals^{|\States|}$, and estimate a ``finite volume'' of $\reachable_\initBel$ via the dual concepts of a \emph{$\delta$-packing} or \emph{$\delta$-covering number}.
While these are theoretical quantities, they can be explicitly computed in certain cases and highlight key properties relating to the \pomdp's complexity~\cite{easypomdp}.

The interested reader can refer to Sect. 1 of the Supplementary Material for a more thorough review of their formal definitions and properties.
In words, the \emph{$\delta$-covering number} $\covNum_{\delta}(\reachable_\initBel)$ is the minimum number of balls of radius $\delta$ needed to cover the set $\reachable_\initBel$.
If in addition, all the centres of the balls are required to belong to $\reachable_\initBel$ then we call such a number the \emph{internal $\delta$-covering number} and denote it by $\covNumInt_{\delta}(\reachable_\initBel)$.
The \emph{$\delta$-packing number} $\pakNum_{\delta}(\reachable_\initBel)$ is the maximum number of points that can be packed inside $\reachable_\initBel$ such that all points are at least $\delta$ distance apart.
The concepts are closely related and, importantly, are finite if and only if $\reachable_\initBel$ is \emph{totally bounded} (see Remark 1 in Supplementary Material).
For instance, it suffices to assume that $\States$ is finite.
To ensure the $\delta$-covering number is always finite, we will make the following standing assumption for the remainder of this paper.

\begin{assumption} \label{ass.standing}
    The reachable belief space $\reachable_\initBel$ is totally bounded.
\end{assumption}

\subsection{KL-Penalisation and \pomdp s}

The idea of using KL-penalisation in fully observable \mdp s started with a series of works on \emph{Linearly Solvable \mdp s}~\cite{todorov,todorov09,todorov09a,dvijotham2012linearly}.
The main idea is to find a control conditional distribution $p(\nst \, | \, \sta)$ to a stochastic control problem where the control cost increases with the relative entropy between $p(\cdot \, | \, \sta)$ and some benchmark $\bar{p}(\cdot \, | \, \sta)$.
The formulation results in a Bellman backup which can be optimised analytically and yields efficient methods to solve a special class of fully-controllable \mdp s.

These works were reformulated over stochastic actions by~\cite{rawlik2012stochastic} and related to general \mdp s by~\cite{azar11,azar12} who introduced \emph{Dynamic Policy Programming}.
This can be interpreted as a policy iteration scheme where each iterate $\pol_k$ is a solution to a specialised \mdp~whose reward decreases with the relative entropy $\KL(\pol_{k+1} \, \| \pol_k)$.
The scheme can be shown to converge to the solution of the \mdp; indeed, the gradual update forced by the KL-penalty yields performance bounds which depend on the \emph{average} accumulated error as opposed to the usual maximum, suggesting robustness to approximation errors.

The extension of the idea of KL-penalised \mdp s to  \pomdp s~was provided by~\cite{kkk23} who introduced the concept of a \emph{\rbpomdpLong} (\rbpomdp).
In essence, the formulation can be viewed as a \emph{Belief-\mdp}~\cite{kaebling98} with policies $U(\bel' \, | \, \bel)$ that control transitions between \pomdp~beliefs where the reward is penalised by the relative entropy $\KL \big( U(\cdot \, | \, \bel) \, \| \, \bar{U}(\cdot \, | \, \bel) \big)$ for some \emph{reference policy} $\bar{U}$.
Their empirical results suggest that approximate solvers for \rbpomdp s can outperform state-of-the-art benchmarks on large \pomdp s for certain choices of $\bar{U}(\cdot \, | \, \bel)$.
However, the authors did not provide a systematic procedure to determine this choice.
This current work addresses this gap by providing a systematic procedure, in a similar vein to \cite{azar12},  

%% file: algorithms.tex
\nos~is an anytime online \pomdp~solver which approximates the solution of a policy iteration scheme whose successive policies are forced to be close to each other.
Specifically, each policy iterate is a solution to a \rbpomdp~over stochastic actions whose reference policy is the previous policy in the sequence and can therefore be viewed as a KL-constrained policy improvement.
While this procedure converges more slowly, its advantage is that it yields a performance bound which is given by the \emph{average} of approximation errors, suggesting that it is less prone to over-commitment---a useful feature given the scarcity of samples generated by an online planner.
In what follows, $\|\cdot \|_\infty$ denotes the usual \emph{supremum-norm} for bounded functions.

\subsection{RBPOMDPs over Stochastic Actions}

In \cite{kkk23}, the reliance on \emph{belief-to-belief transitions} $U(\bel' \, | \, \bel)$ implicitly allows the agent to control the choice of observation, which may not be valid in general.
We will consider a more natural formulation over \emph{stochastic actions} which will form the building blocks for the required systematic procedure.
Namely, a \emph{\rbpomdp~over stochastic actions} is a tuple $\langle \States, \Actions, \Observations, \TP, \OP, \Reward, \discount, \temp, \refPol, \initBel\rangle$. Its value $\refVal$, for a given $\bel \in \reachable_\initBel$, is specified by the recursive equation
\begin{multline} \label{eq.ref.bellman}
    \refVal(\bel) = \sup_{\pol \in \policies} \Big[ \sum_{\act \in \Actions} \Reward(\bel, \act) \pol(\act \, | \, \bel) - \frac{1}{\temp} \KL( \pol \, \| \, \refPol ) \\
    + \discount \sum_{\act, \obs} \Prob(\obs \, | \, \act, \bel) \pol(\act \, | \, \bel) \refVal\big( \tau( \bel, \act, \obs) \big) \Big].
\end{multline}
Intuitively its solution is a stochastic policy that tries to respect the reference policy $\refPol$ unless deviating substantially leads to greater rewards where the trade-off is balanced by the temperature parameter $\temp > 0$.
The right-hand-side can be optimised analytically so that \eqref{eq.ref.bellman} is equivalent to
\begin{multline} \label{eq.maximised}
    \refVal(\bel) = \frac{1}{\temp} \log \Big[\sum_{\act \in \Actions} \refPol(\act \, | \, \bel) \exp \Big\{ \temp \big[ \Reward(\bel, \act) \\
    + \discount \sum_\obs \Prob(\obs \, | \, \act, \bel) \refVal\big( \tau (\bel, \act, \obs) \big) \big] \Big\} \Big].
\end{multline}
In fact, we can represent the Bellman equation \eqref{eq.maximised} in a slightly different way by introducing \emph{preferences} $\actPref$ over belief-action pairs.
More specifically, let
\begin{multline}
    \actPref(\bel, \act) := \frac{1}{\temp} \log \big( \refPol(\act \, | \, \bel) \big) + \Reward(\bel, \act) \\
    + \discount \sum_\obs \Prob(\obs \, | \, \act, \bel) \refVal\big( \tau (\bel, \act, \obs) \big).
\end{multline}
This yields
$
    \refVal(\bel) = \frac{1}{\temp} \log \Big[ \sum_\act \exp[\temp \actPref(\bel, \act)] \Big] =: [\logOp_\temp \actPref](\bel)
$
where $\logOp_\temp$ is the \emph{log-sum-exp} operator~\cite{bhh2021,al2017} and eq. \eqref{eq.maximised} stated with respect to preferences becomes
\begin{multline}
\label{eq.log.maximised}
    \actPref(\bel, \act) = \frac{1}{\temp} \log[\refPol(\act \, | \, \bel)] + \Reward(\bel, \act) \\+ \discount \sum_\obs \Prob(\obs \, | \, \act, \bel) [\logOp_\temp \actPref] \big( \tau (\bel, \act, \obs) \big).
\end{multline}
If $\actPref^*$ satisfies \eqref{eq.log.maximised}, the solution of the \rbpomdp~is
\begin{equation} \label{eq.rbpomdp.soln}
\optRefPol(\act \, | \, \bel) = \frac{\exp[\temp \actPref^*(\bel, \act) ]}{\sum_{\act'} \exp[\temp \actPref^*(\bel, \act') ]}.
\end{equation}
being the exact maximiser of \eqref{eq.ref.bellman}.

\subsection{Exact Scheme}
\label{subsec.exact.scheme}

We are now in a position to describe the exact iterative scheme that relates the \rbpomdp~to that of the standard \pomdp.
Taking inspiration from \eqref{eq.rbpomdp.soln}, the scheme implicitly represents a reference policy $\pol_k$ by maintaining \emph{action preferences} $\actPref_k: \reachable_\initBel \times \Actions \rightarrow \Reals$ according to the equation
\begin{equation} \label{eq.pol}
    \pol_{k}(\act \, | \, \bel) := \frac{\exp[\temp \actPref_{k}(\bel, \act) ]}{\sum_{\act'} \exp[\temp \actPref_{k}(\bel, \act') ]}.
\end{equation}
The policy is then updated \emph{gradually} by asserting that $\pol_{k+1}$ is the solution to a \rbpomdp~whose reference policy is $\pol_k$. That is,
\begin{multline} \label{eq.scheme}
    \actPref_{k+1}(\bel, \act) = \frac{1}{\temp} \log[\pol_k(\act \, | \, \bel)] + \Reward(\bel, \act) \\+ \discount \sum_\obs \Prob(\obs \, | \, \act, \bel) [\logOp_\temp \actPref_k] \big( \tau (\bel, \act, \obs) \big) \\
    = \actPref_k(\bel, \act) - [\logOp_\temp \actPref_k](\bel) + \Reward(\bel, \act) \\ + \discount \sum_\obs \Prob(\obs \, | \, \act, \bel) [\logOp_\temp \actPref_k]\big( \tau(\bel, \act, \obs) \big)\\
    =: [\DPP_\temp \actPref_k] (\bel, \act).
\end{multline}
The exact scheme indeed converges to the action-value $\optQVal$ of the \pomdp.
To show this, let $\DPP_\temp$ be the exact function operator defined by \eqref{eq.scheme} and consider a sequence of \emph{approximate preferences} $(\approxAP_k)_{k \ge 0}$ such that $\approxAP_{k+1} \approx \DPP_\temp \approxAP_k$. 
For arbitrary $(\bel, \act) \in B \times \Actions$, let
\begin{equation}
    \error_k(\bel, \act) := 
    \begin{cases}
        \approxAP_{k}(\bel, \act) - [\DPP_\temp \approxAP_{k-1}] (\bel, \act) & \text{if } k \ge 1\\
        0 & \text{if } k = 0
    \end{cases}
\end{equation}
and
$
    E_k(\bel, \act) := \sum_{j = 0}^k \error_j(\bel, \act)
$
and define the approximating policy to be
\begin{equation} \label{eq.approx.pol}
    \approxPol_{k}(\act \, | \, \bel) :=  \frac{\exp[\temp \approxAP_{k}(\bel, \act) ]}{\sum_{\act'} \exp[\temp \approxAP_{k}(\bel, \act') ]}.
\end{equation}
We have the following general error bound which says that the total error is bounded by the \emph{average} of approximation errors at each iteration.
Since the exact scheme has $E_k = 0$ for all $k$, the result also validates the asymptotic convergence of the exact scheme.

\begin{theorem} \label{thm.approx.bound}
Suppose $\| \approxAP_0 \|_\infty \le \Vmax$.
Then
\begin{multline}
        \| \optQVal - \QVal^{\approxPol_k} \|_{\infty} \\
        \le \frac{2}{(1 - \discount)(k + 1)}\Big[  \frac{\discount (4 \Vmax + \frac{\log(|\Actions|)}{\temp})}{(1 - \discount)} \\
        + \sum_{j=0}^k \discount^{k-j} \| E_j \|_\infty  \Big].
\end{multline}
\end{theorem}

\begin{proof}
    See Supplementary Material.
\end{proof}

\subsection{Explicit Sampling-Based Approximate Scheme}
\label{subsec.approx.scheme}

We will now introduce explicit synchronous and asynchronous sampling-based approximate schemes and prove their asymptotic optimality.
In both cases, we prove specialised bounds with respect to the \pomdp's $\delta$-covering numbers.
The asynchronous scheme is especially important, as it forms the basis for the design of our online planning algorithm.

We will need some setting up to introduce the sampling-based scheme that approximates \eqref{eq.scheme}.
Let $\dist$ be a metric on $\belSpace$ and $B$ be some well-ordered\footnote{It suffices for $B$ to be finite.} subset of $\belSpace$.
Let $\nearestBel_{B, \dist}: \belSpace \times \Actions \times \Observations \rightarrow B$ be the mapping which takes an arbitrary belief $\bel \in \belSpace$ to the least element of
\begin{equation} \label{eq.proj}
    \argmin_{\bel' \in B} \dist\big(\bel', \tau(\bel, \act, \obs)\big).
\end{equation}
Intuitively, $\nearestBel_{B, \dist}$ finds the set of points in $B$ nearest to $\tau(\bel, \act, \obs)$ (it is not necessarily a singleton set) and has a rule to break ties so that the mapping is well-defined.

Let $\QVal^\pol_{B, \dist}: \reachable_\initBel \times \Actions \rightarrow \Reals$ be the \emph{action-value approximation} on any subset $B \subset \reachable_\initBel$ which is the unique solution to the recursion
\begin{multline} \label{eq.qval.snap}
        \QVal^\pol(\bel, \act) = \Reward(\bel, \act) \\
        + \discount \sum_{\act', \obs} \QVal^\pol\big(\nearestBel_{B, \dist}(\bel, \act, \obs), \act'\big) \Prob(\obs \, | \, \act', \bel) \, \pol(\act' \, | \, \bel).
\end{multline}
The difference between \eqref{eq.qval} and \eqref{eq.qval.snap} is that the next belief is forced to a nearest belief in $B \subset \reachable_\initBel$ in the latter, whereas the belief update for the former is the natural one. As such, we expect the two quantities to differ according to the precision of $B$ in approximating $\reachable_\initBel$. 
In fact, it can be shown that if $B$ is a $\delta$-covering of $\reachable_\initBel$ the approximation becomes negligible for the optimal policy $\optPol$ as $\delta \downarrow 0$ (see Proposition 3 in the Supplementary Material).

It is clear from \eqref{eq.qval.snap} that it suffices to evaluate $\QVal^\pol_{B, \dist}$ on the subset $B \times \Actions$.
The \emph{synchronous} scheme therefore updates action preference approximations according to the rule
\begin{multline} \label{eq.sbas}
    \approxAP_{k+1}(\bel, \act) := \approxAP_k(\bel, \act) - [\logOp_\temp \approxAP_k] (\bel) 
     + \sum_{i=1}^{N_k(\bel, \act)} \frac{\Reward(\sta_i, \act)}{N_k(\bel, \act)} \\
     + \discount \sum_{j=1}^{M_k(\bel, \act)} \frac{[\logOp_\temp \approxAP_k]\big(\nearestBel_{B, \dist} (\bel, \act, \obs_j) \big)}{M_k(\bel, \act)}
\end{multline}
for all $(\bel, \act) \in B \times \Actions$ where $\sta_i \sim \bel$ and $\obs_j \sim \Prob(\cdot \, | \, \act, \bel)$ and generic increasing sequences $N_k$ and $M_k$ having the property that $N_k(\bel, \act) \uparrow \infty$, $M_k(\bel, \act) \uparrow \infty$ as $k \uparrow \infty$.
The scheme is \emph{synchronous} in the sense that, at each step $k$, it samples $\{\sta_{N_{k-1}(\bel, \act) + 1}, \ldots, \sta_{N_k(\bel, \act)}, \obs_{M_{k-1}(\bel, \act) + 1}, \ldots, \obs_{M_k(\bel, \act)}\}$ \emph{for each} $(\bel, \act)$ and updates the action preferences according to \eqref{eq.sbas}.
The approximate stochastic policy $\approxPol_k$ is then fully specified by the approximate preferences according to \eqref{eq.approx.pol}.

The synchronous scheme yields the following high-probability bound when $B$ is an internal $\delta$-covering $\covSet_\delta$ of $\reachable_\initBel$ for the metric $\dist_1$ induced by the 1-norm---i.e. $\dist_1(x,y) := \| x - y \|_1$ for $x, y \in \Reals^{|\States| - 1}$.\footnote{Note that the Euclidean space under consideration can, in theory, be infinite-dimensional under Assumption \ref{ass.standing}.}

\begin{theorem} \label{thm.exp.bd}
Let $\covNumInt_\delta = |\covSet_\delta|$ be the internal $\delta$-covering number of $\reachable_\initBel$ for a given $\delta > 0$.
If $\| \approxAP_0 \| \le \Vmax$ then, for any $\alpha \in (0,1)$, we have with probability at least $1 - \alpha$
\begin{equation} \label{eq.exp.bd}
    \| \optQVal - \QVal^{\approxPol_k}_{\covSet_\delta, \dist_1} \|_\infty \le \frac{K_1}{k+1} + \frac{K_2}{\sqrt{k+1}} + \frac{\discount \delta \Vmax}{1 - \discount}
\end{equation}
where
\begin{equation}
    K_1 := \frac{2\discount}{(1- \discount)^2} \big[\log(|\Actions|)/\temp + 4 \Vmax \big]
\end{equation}
and
\begin{equation}
    K_2 := \Big[ \frac{4 \discount \log(|\Actions|)}{\temp (1-\discount)^3} + \frac{2 \Vmax}{1 - \discount} \Big] \sqrt{2 \log \Big\{ \frac{2 |\Actions| \covNumInt_\delta}{\alpha}\Big\}}.
\end{equation}
\end{theorem}

\begin{proof}
See Supplementary Material.
\end{proof}

Although the precision of the bound gets more precise after every synchronous update, the error can still be large if the covering $\covSet_\delta$ is not a good representation of $\reachable_\initBel$---i.e. $\delta$ is large. In general, $\covSet_\delta$ may be required to be extremely large and performing even one synchronous update can be an exorbitantly expensive task. 

To mitigate this fundamental problem, \nos~employs a heuristic action sampler $\sampler$ to bias towards a selection of promising beliefs and asynchronously updates preference approximations on the selection. The underpinning assumption for optimality of this procedure is that the selection grows to include the set of beliefs reachable under the optimal policy $\optPol$---which is not known a priori---while simultaneously being small enough to be tractable for online planning.

More precisely, let $\covSet_\delta$ be an internal $\delta$-covering of $\reachable_\initBel$ and let $\subsetBA_k :=   \big((\bel_1, \act_1), (\bel_2, \act_2), \ldots, (\bel_k, \act_k)\big)$ be the sequence of pairs in $\covSet_\delta \times \Actions$ traversed by $\sampler$ after $k$ steps.
Then, by definition, our \emph{asynchronous} scheme updates action preference approximations according to
\begin{multline} \label{eq.async}
    \approxAP_{k+1}(\bel_k, \act_k) := \approxAP_k(\bel_k, \act_k) \\
    - [\logOp_\temp \approxAP_k] (\bel_k) 
     + \sum_{i=1}^{N(\bel_k, \act_k)} \frac{\Reward(\sta_i, \act_k)}{N(\bel_k, \act_k)} \\
     + \discount \sum_{j=1}^{N(\bel_k, \act_k)} \frac{[\logOp_\temp \approxAP_k]\big(\nearestBel_{\covSet_\delta, \dist_1} (\bel_k, \act_k, \obs_j) \big)}{N(\bel_k, \act_k)}
\end{multline}
where $\sta_i \sim \bel_k$ and $\obs_j \sim \Prob(\cdot \, | \, \act_k, \bel_k)$ and $N(\bel_k, \act_k)$ is the number of times $\sampler$ has visited $(\bel_k, \act_k)$.
Let $\optReachable_\initBel$ be the set of beliefs reachable under the optimal policy $\optPol$ of the \pomdp~and denote by $\kappa_k$ the number of times that $\sampler$ has visited $\subsetBA_\infty$ after $k$ steps.
Then, provided $\sampler$ traverses $\subsetBA_\infty$ infinitely often and $\{ \bel : (\bel, \act) \in \subsetBA_\infty \} \supset \optReachable_\initBel \cap \covSet_\delta$, the bounds of Theorem \ref{thm.exp.bd} hold
with \eqref{eq.exp.bd} replaced by
\begin{equation}
    \| \optQVal - \QVal^{\approxPol_{k}}_{\covSet_\delta, \dist_1} \|_\infty \le \frac{K_1}{\kappa_k+1} + \frac{K_2}{\sqrt{\kappa_k+1}} + \frac{\discount \delta \Vmax}{1 - \discount}
\end{equation}
for $\optQVal$ and $\QVal^{\approxPol_k}_{\covSet_\delta, \dist_1}$ being functions defined on $\subsetBA_\infty$ and $\covNumInt_\delta$ now being the $\delta$-covering number of $\subsetBA_\infty$.
As such, we would like to ensure that $\subsetBA_\infty$ is as small as possible without compromising optimality.

\subsection{Algorithm: \nos}

We propose \nosLong~(\nos), a specific online implementation of the asynchronous scheme discussed above.
\nos~represents beliefs as nodes $\hist$ in a tree where each node is associated with a history of action-observation pairs and maintains a belief estimate by progressively sampling a richer set of state particles at each node.
With enough time, the planner grows a rich tree (i.e. $\delta$ small) and improves preference estimates by sampling sequences of action-observation histories up to a required depth $\maxDepth$ and backing up estimates according to the sampling-based scheme.

\begin{algorithm}[tbp]
    \caption{\nos}
    \label{alg.porpi}
    \textbf{Input}: Root node $\hist_0$ of $\Tree$ equipped with belief particles $\bel$\\
    \textbf{Output}: $\Tree$
    \begin{algorithmic}[1] 
    \WHILE{steps remaining}
        \WHILE{time permitting}
            \STATE Sample belief particle $\sta$ from $\hist_0$
            \STATE $\textsc{Simulate}(\hist_0, \sta, 0)$
        \ENDWHILE
        \STATE $\mact \gets \argmax_{\mact \in \text{children}(\hist_0)} \approxAP(\hist_0 \mact)$
        \STATE Execute macro action $\mact$ in environment
        \STATE Receive macro observation $\mobs$ from environment
        \STATE Update history $\hist_0 \gets \hist_0 \mact \mobs$
        \STATE Resample new state particles and add to $\hist_0$
    \ENDWHILE
    \end{algorithmic}
\end{algorithm}

Specifically, at each history, \nos's heuristic action sampler \textsc{SampleCandidateAction}$(\hist, \sta)$ uses domain specific knowledge about the problem to propose a (macro) action $\mact$---i.e. a sequence of primitive actions---to add to the tree.
The aim of the sampler is to sample actions that cover the optimal policy while avoiding counterproductive ones---see Sec. \ref{sec.heur} for examples.
The action is added to the tree if it has not been already and the progressive widening threshold $\AWF N(\hist)^\AWA$ (e.g. \cite{sunberg}) has not been exceeded.
\nos~then selects an action 
by sampling the softmax distribution given by the current preferences---cf. \eqref{eq.approx.pol}---before sampling a new state $\nst$, (macro) observation $\mobs$ and (macro) reward $r(\sta, \mact; \discount)$ using a generative model.
The observation is then added to the tree, and the procedure continues recursively until the depth exceeds $\maxDepth$. At this point the value is estimated from the sampled state using a value heuristic and the information is propagated back up to the root node via \eqref{eq.async} (lines 18 to 23 in Algorithm \ref{alg.simulate}).

\begin{algorithm}[tbp] 
    \caption{$\textsc{Simulate}(\hist, \sta, \depth)$} 
    \label{alg.simulate}
    \textbf{Parameters}: {$\AWF\ge0, \AWA\in(0,1), \maxDepth\ge1, \temp > 0$.}\par\vspace{-0.em}
    \begin{algorithmic}[1]
    \IF {$\depth > \maxDepth$}
        \RETURN $\textsc{ValueHeuristic}(\hist, \sta)$
    \ENDIF
    \IF {$\depth > 0$}
        \STATE $\bel(\hist) \gets \bel(\hist) \cup \{\sta\}$
    \ENDIF
    \STATE $N(h) \gets N(h) + 1$

    \IF {$|\text{children}(\hist)| < \AWF N(h)^\AWA$}
        \STATE $\mact \gets \textsc{SampleCandidateAction}(\hist, \sta)$
        \IF {$\hist\mact \notin \Tree$}
            \STATE Add $\hist\mact$ to $\Tree$
        \ENDIF
    \ENDIF
    \STATE $\mact \gets \textsc{SamplePrefSoftMax}(\hist; \temp)$
    \STATE Resample $\sta$ from $\bel(\hist)$
    \STATE Sample $(\nst, \mobs, r(\sta, \mact; \discount))$ from gen. model $\GenModel(\sta, \mact)$
    \STATE Create nodes for $\hist \mact \mobs$ if not created already
    \STATE $N(\hist \mact) \gets N(\hist \mact) + 1$
    \STATE $R(\hist \mact) \gets R(\hist \mact) + \frac{r(s, \mact; \discount) - R(\hist \mact)}{N(\hist \mact)}$
    \STATE $D(\hist \mact) \gets D(\hist \mact) + \frac{\textsc{Simulate}(\hist \mact \mobs, \nst, \depth + |\mact|) - D(\hist\mact)}{N(\hist \mact)}$
    \STATE $\approxAP(\hist \act) \gets \approxAP(\hist \act) - \refVal(\hist) + R(\hist \mact) + \discount^{|\mact|} D(\hist \mact)$
    \STATE $\refVal(\hist) \gets \log \big\{ \sum_{\act \in \text{children}(\hist)} \exp[\temp \approxAP (\hist \act)] \big\} / \temp$
    \RETURN $\refVal(\hist)$
    \end{algorithmic}
\end{algorithm}

This planning procedure continues until timeout (lines 2 to 5 in Algorithm \ref{alg.porpi}) after which the algorithm executes the action with the best sampled preference in the environment.
Upon receiving an observation, particles that are consistent with the realised action-observation pair are resampled and added to the associated node (line 10 in Algorithm \ref{alg.porpi}).
This planning-execution loop continues until a step budget is reached, at which point the algorithm terminates.

%% file: experiments.tex
\subsection{Problem Scenarios}

We evaluated the performance of \nos~on two challenging long-horizon \pomdp s.

\paragraph{3D Maze with Poor Localisation.}
A 3-dimensional holonomic cuboid drone needs to navigate to one of two goal regions in a closed maze with very poor localisation (Figure \ref{fig.urban}).
The state of the robot is represented by a continuous 3-dimensional co-ordinate for its centre of mass, and the robot can move continuously in any direction of fixed magnitude (i.e. $v = 1$) plus some mean zero (Gaussian) noise with covariance matrix $\mathbf{I} \times 0.02 \times v$ and any movement conforms to the ``walls'' of the environment.
However, the robot does not know its true state and only knows that it can spawn at two starting positions with equal probability (Figure \ref{fig.urban}).
The robot can only localise its co-ordinate if it comes in contact with a landmark where it receives an observation of its true position; otherwise, it receives no feedback about its position.
The scenario terminates if the robot comes in contact with a danger zone---which incurs a penalty of -500---or reaches the goal---which yields a reward of 2000.
A step penalty of -5 is incurred in all other cases.
This is a long-horizon problem requiring 100 steps to reach the goal while simultaneously navigating around danger zones.

\begin{figure}[tb]
    \centering
    \includegraphics[width=0.95\linewidth]{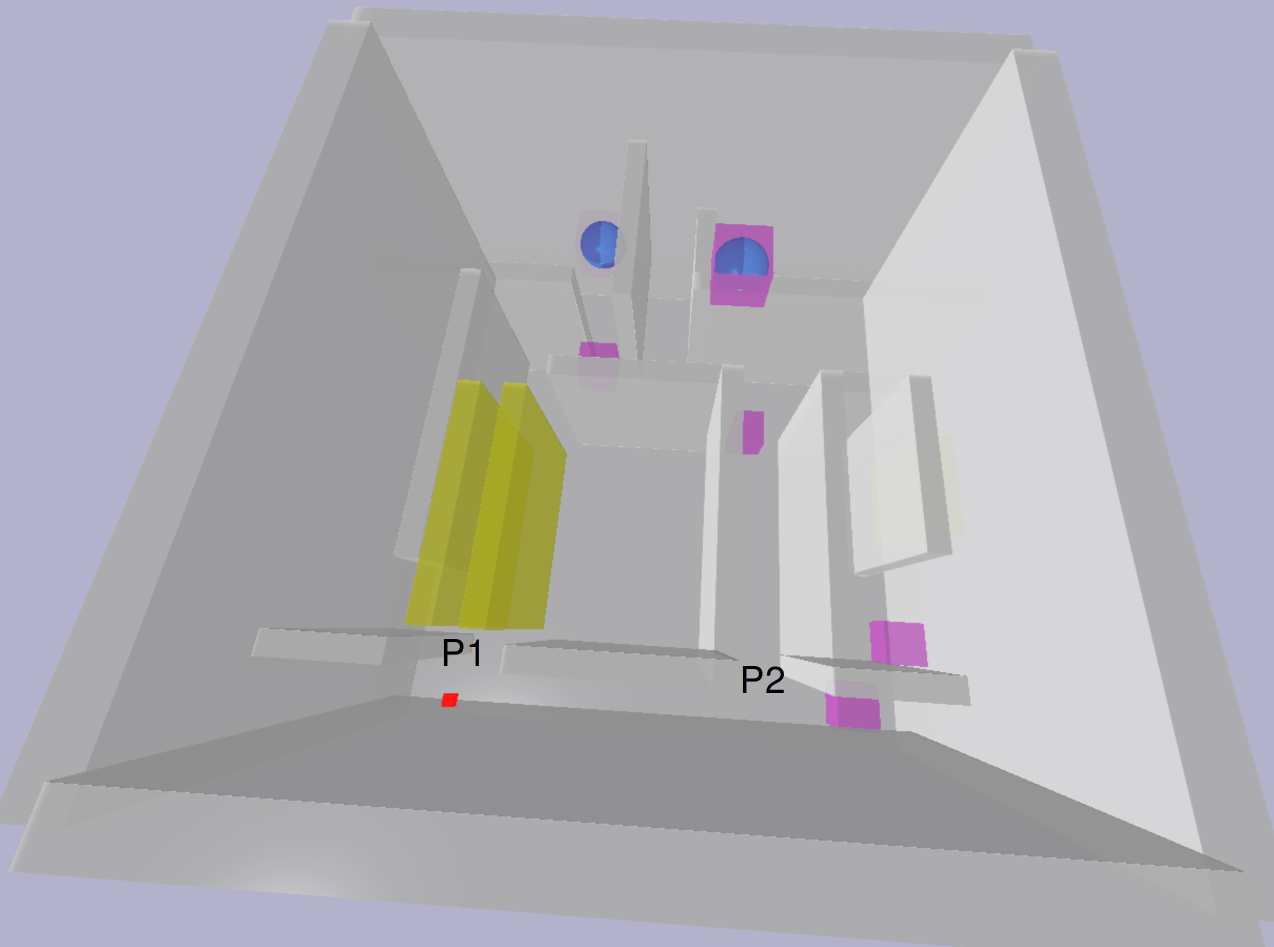}
    \caption{\emph{3D Maze with Poor Localisation}. The environment is a closed box and walls are \textbf{\textcolor{Gray}{grey}}; danger zones are \textbf{\textcolor{Dandelion}{yellow}}; landmarks are \textbf{\textcolor{Mulberry}{purple}}; goal region is labelled \textbf{\textcolor{Blue}{blue}}.
    The robot spawns in two positions P1 and P2 with equal likelihood and the robot does not receive any initial feedback about its position.
    If the robot spawns at P1, the direct route to the goal has a high likelihood of collision with a danger zone so the robot must localise first and take a safer route.} 
    \label{fig.urban}
\end{figure}

\paragraph{HEMS Mission with Evolving No-Fly-Zones.} 
We considered a Helicopter Emergency Medical Service (HEMS) mission set on the Cap Corse peninsula in Corsica (Figure \ref{fig.corsica}). The mesh used to generate the terrain was extracted from X-Plane 12.
The mission objective is to navigate a holonomic helicopter starting from the west end of the island (arrow in Figure \ref{fig.corsica} (a)) to two unordered objectives---i.e. the victim's locations (green balls in Figure \ref{fig.corsica}) where the agent receives a reward of 2000 for each new objective achieved.
The mission ends if there is a collision (which incurs a reward penalty of -2000) or both objectives are achieved---i.e. the mission is accomplished---which yields an additional reward of 20000.
The scenario is complicated by the fact that no-fly-zones (NFZs) evolve at fixed time steps that are unknown to the agent (see Figure \ref{fig.corsica}).
The agent need not avoid NFZs entirely, but incurs an additional penalty of -20 for each step inside a NFZ.
We assume that the agent has no predictive model of when NFZs will appear; hence, the agent only re-plans with respect to reward changes due to NFZ evolutions.
To encourage the agent to achieve the objective, a step penalty of -5 is incurred at each time step.
The state of the helicopter is fully specified by a continuous 3-dimensional co-ordinate representing the helicopter's centre of mass (its orientation is always fixed)---notice that fuel and weight of the craft are not considerations---and actions are the continuous directional vectors of a fixed magnitude $v = 2$ (i.e. the helicopter's speed) representing the agent's intended direction of movement.
Transitions in the intended direction and readings of the true state of the helicopter are subject to Gaussian noise with covariance matrices $\mathbf{I} \times 0.25 \times v$ and $\mathbf{I} \times 0.2$ respectively.
This problem is a long-horizon problem often requiring a minimum of 150 steps to accomplish the mission without consideration of NFZs.

\begin{figure}[tbp]
    \centering
    \begin{subfigure}[b]{0.23\textwidth}
        \centering
        \includegraphics[width=\textwidth]{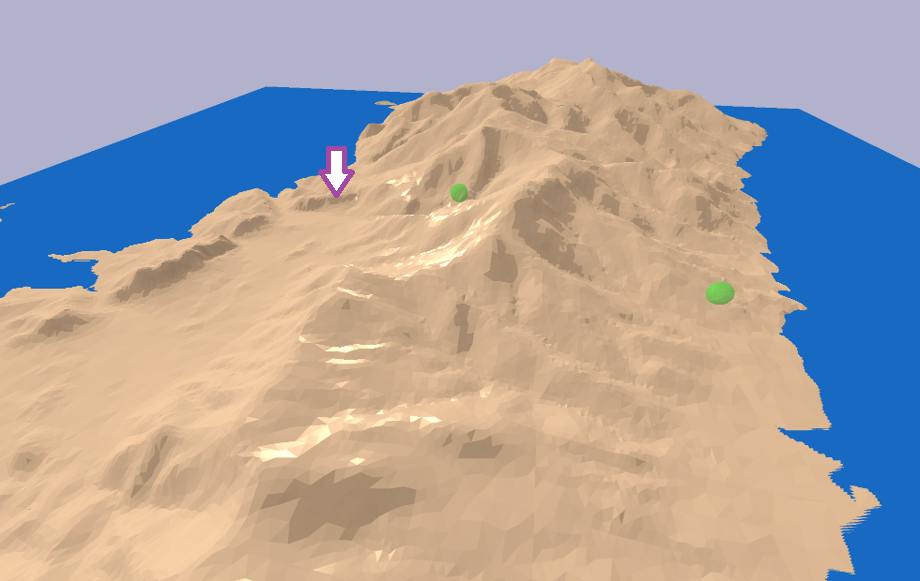}
        \caption
        {{\small Steps 1--14}}    
        \label{fig.corsica1}
    \end{subfigure}
    \begin{subfigure}[b]{0.23\textwidth}  
        \centering 
        \includegraphics[width=\textwidth]{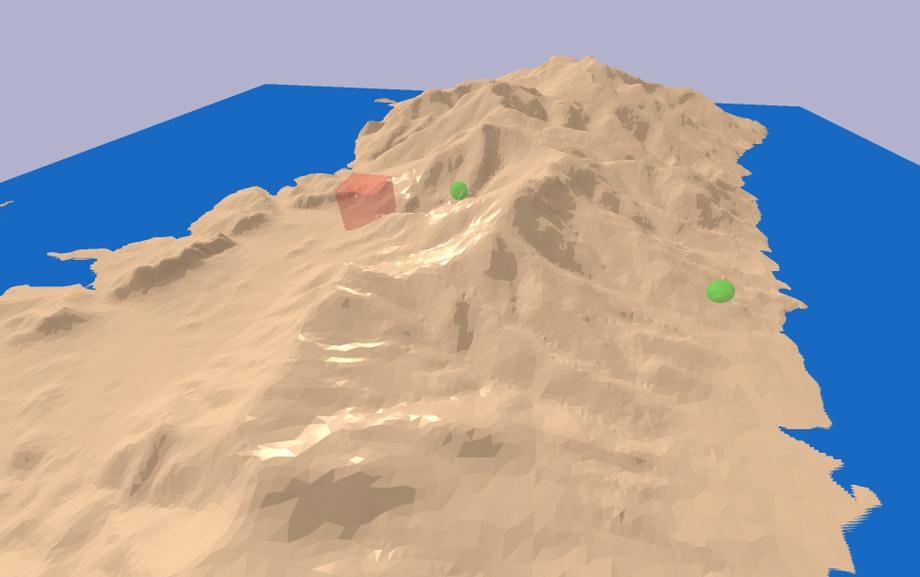}
        \caption
        {{\small Steps 15--49}}    
        \label{fig.corsica2}
    \end{subfigure}
    \begin{subfigure}[b]{0.23\textwidth}   
       \centering 
        \includegraphics[width=\textwidth]{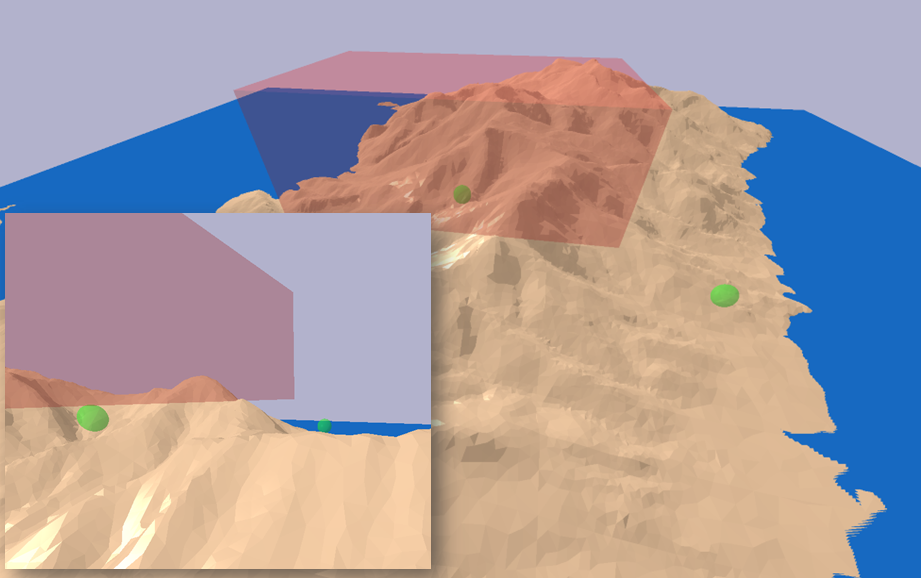}
        \caption
        {{\small Steps 50--99}}   
        \label{fig.corsica3}
    \end{subfigure}
    \begin{subfigure}[b]{0.23\textwidth}   
        \centering 
        \includegraphics[width=\textwidth]{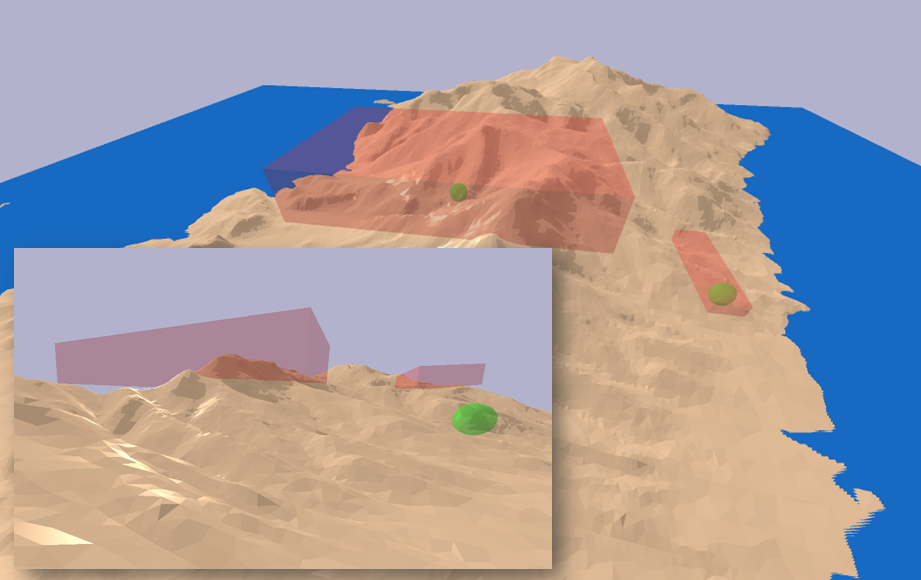}
        \caption
        {{\small Steps 99+}}   
        \label{fig.corsica4}
    \end{subfigure}
    \caption[]
    {Corsica Rescue Mission with Evolving NFZs. The starting position is indicated by the arrow in (a); objectives are \textbf{\textcolor{Green}{green}}; NFZs are \textbf{\textcolor{red}{red}}. The environment evolves at preset time-steps that are unknown to the agent. The agent should react to avoid NFZs but may elect not to do so in order to evade a greater catastrophe.}
    \label{fig.corsica}
\end{figure}

\subsection{Heuristic Action Sampler}
\label{sec.heur}

One crucial factor in the overall performance of \nos~is the heuristic action sampler \textsc{SampleCandidateAction}$(\hist, \sta)$.
We stress that the heuristic action sampler is \emph{not} a solution to the \pomdp; indeed, the heuristic sampler need not account for uncertainty being a function of a determined state.
Rather, its fundamental purpose is to exploit domain-specific knowledge to propose promising actions to explore given a belief.

In both environments, our specific implementation of this subroutine relies on an offline-generated Probabilistic Roadmap (PRM) \cite{choset} to represent the environment's collision-free configuration space.
Based on the input particle an \emph{objective} in the environment's configuration space is sampled and collision-free paths to the sampled objective are queried from the PRM. That is,
\begin{itemize}
    \item For the 3D maze, a random landmark or goal region is sampled and targeted and the shortest path on the PRM starting from the position given by the state particle to the target is returned.
    \item For the Corsica map, the state $\sta$ records which victims have been visited.
    Accordingly, a simple homotopic collision-free path starting from the helicopter's position (as recorded in $\sta$) and ending at a random unvisited victim location is sampled. 
\end{itemize}
The returned paths are then truncated at a fixed length, and a macro action which traces the path is returned.

\subsection{Benchmark Methods}

The benchmarks used for comparison are:
\begin{itemize}
    \item \textit{RefPol}. This simply samples a state particle and executes the action returned by the heuristic action sampler without further \pomdp~planning.
    \item \textit{RefSolver}. The solver from \cite{kkk23} a \rbpomdp~which uses the heuristic action sampler as its \emph{reference policy}.
    \item \textit{POMCP}---\cite{sv10}. The canonical benchmark to beat for online POMDP planning. For a fair comparison, it expands 16 macro actions composed of equally spaced directional vectors.
\end{itemize} 

\subsection{Experimental Setup}

All experiments were performed on a desktop computer with 128GB DDR4 RAM and an 8 Core Intel Xeon Silver 4110 Processor.
All solvers were implemented in the pomdp\_py library \cite{h2r} and Cythonised for a fair comparison.
The discount factor for all environments was $\discount = 0.99$.\footnote{See \url{https://github.com/RDLLab/pomdp-py-porpp} for the code and parameters used to run the experiments.}

\subsection{Results and Discussion}

Results are summarised in Table \ref{tab:3dmaze} and Table \ref{tab:corsica}.
In both scenarios we ran RefPol to corroborate our claim that the heuristic action sampler is significantly sub-optimal.
Still, \nos~was able leverage the heuristic action sampler to significantly outperform both benchmarks yielding very high success rates with $>$10 seconds of planning time. As expected from our theoretical analysis, the results improve in trend with the planning time.
Notably, RefSolver does not improve quite as much \nos which seems consistent with the idea that RefSolver is converging to a policy which is somewhere in between the reference policy and the optimal policy of the \pomdp.
POMCP, meanwhile, was myopic in both scenarios and could not take advantage of deep rewards even when helped by macro actions because of the need to exhaustively enumerate.
Interestingly, in the HEMS mission, we typically observe the \nos~policy trace non-trivially adapting to the environment (Figure \ref{fig.deform}).

\begin{table}[tb]
    \centering
    \begin{tabular}{lrrr}
        \toprule
        Planners & Time (s) & Succ. \% & E[Tot. Reward] \\
        \midrule
        \nos & 1 & 71 & 570.3 $\pm$ 183.5 \\
            & 2 & 75 & 628.9 $\pm$ 191.3 \\
            & 3 & 80 & 625.0 $\pm$ 215.2 \\
            & 5 & 81 & 688.3 $\pm$ 200.1 \\
            & 10 & 88 & 873.4 $\pm$ 172.1 \\
            & 15 & 94 & 983.1 $\pm$ 168.0 \\
    RefSolver & 2 & 39 & -244.6 $\pm$ 224.5\\
              & 3 & 38 & -278.9 $\pm$ 216.5\\
              & 5 & 26 & -544.9 $\pm$ 204.6\\
              & 10 & 30 & -384.0 $\pm$ 213.3 \\
        POMCP & 2 & 10 & -786.3 $\pm$ 378.0 \\
              & 3 & 9 & -1637.3 $\pm$ 256.8 \\
              & 5 & 7 & -2150.8 $\pm$ 161.0 \\
              & 10 & 13 & -1897.5 $\pm$ 240.1 \\
        RefPol & N/A & 29 & -572.2 $\pm$ 231.1 \\
        \bottomrule
        \end{tabular}
    \caption{Results for 3D Maze with Poor Localisation (100 runs; maximum macro action length = 10)}
    \label{tab:3dmaze}
\end{table}

\begin{table}[tb]
    \centering
    \begin{tabular}{lrrr}
        \toprule
        Planners & Time (s) & Succ. \% & E[Tot. Reward] \\
        \midrule
        \nos & 1 & 58 & 11393.5 $\pm$ 1588.4\\
             & 2 & 75 & 15408.8 $\pm$ 1399.3\\
             & 3 & 78 & 16207.7 $\pm$ 1316.7\\
             & 5 & 78 & 16231.6 $\pm$ 1320.2\\
             & 10 & 90 & 19393.5 $\pm$ 967.9\\
    RefSolver & 2 & 2 & -1453.9 $\pm$ 947.8\\
              & 3 & 4 & -860.6 $\pm$ 1297.7\\
              & 5 & 28 & 3514.9 $\pm$ 3043.1\\
              & 10 & 22 & 2258.7 $\pm$ 2809.2\\
        POMCP & 2 & 2 & -410.5 $\pm$ 900.0\\
              & 3 & 0 & -942.5 $\pm$ 181.1\\
              & 5 & 2 & -421.8 $\pm$ 928.1\\
              & 10 & 0 & -839.6 $\pm$ 227.4\\
        RefPol & N/A & 0 & -6584.3 $\pm$ 379.5 \\
    \bottomrule
    \end{tabular}
    \caption{Results for HEMS Mission with Evolving NFZs (100 runs; maximum macro action length = 15)}
    \label{tab:corsica}
\end{table}

\begin{figure}[tb]
    \centering
    \begin{subfigure}[b]{0.225\textwidth}   
       \centering 
        \includegraphics[width=\textwidth]{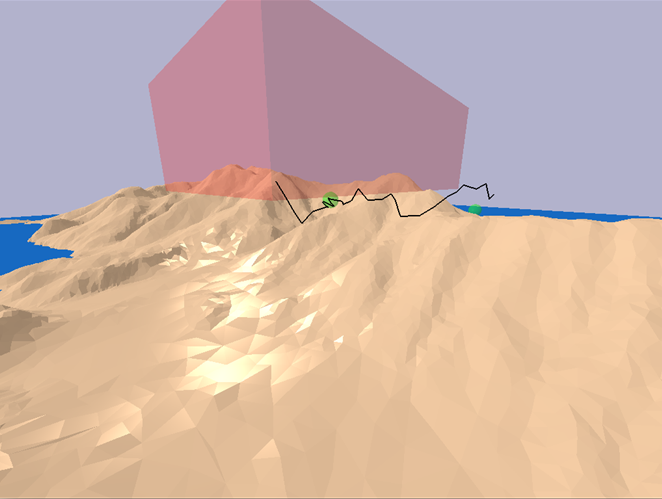}
        \label{fig.deform1}
    \end{subfigure}
    \begin{subfigure}[b]{0.225\textwidth}   
        \centering 
        \includegraphics[width=\textwidth]{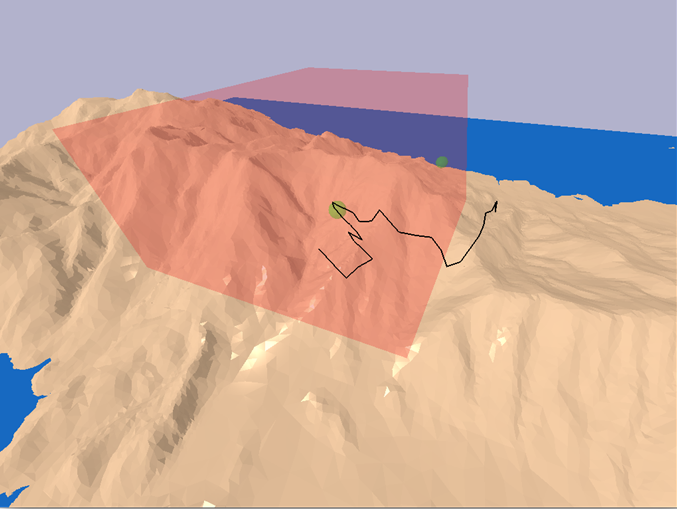}
        \label{fig.deform2}
    \end{subfigure}
    \caption[]
    {Two perspectives of the \nos~trajectory trace of the helicopter in the HEMS mission during steps 50--149. At the beginning of the trace the helicopter initially descends to avoid the new NFZ and targets the nearest objective. Once this objective is achieved, the helicopter successful navigates a path around the NFZ and surrounding terrain rather than taking the shortest path through the NFZ to the next objective.}
    \label{fig.deform}
\end{figure}

%% file: summary.tex
This paper presents \nos~an online particle-based anytime \pomdp~solver which provably approximates a gradual KL-constrained iterative scheme making it robust to large approximation errors.
Empirical results indicate the feasibility of our planner for large-scale \pomdp s showing that it outperforms existing benchmarks for the long-horizon \pomdp s with evolving environments presented in this paper.

For future work, we would like to examine the solver on non-holonomic problems (realistic ODE approximations of helicopter dynamics, robotic manipulators, etc.) with more complex domains (e.g. HEMS fire and flood rescue scenarios).
We would also like to systematically stress test \nos~with respect to different parameter settings and choices of heuristic samplers.